\documentclass[10pt]{amsart}
\usepackage[usenames,dvipsnames]{xcolor}
\usepackage{amsmath,amsthm,amssymb,color,comment,csquotes,enumerate,fancyhdr,filecontents,graphicx,stmaryrd,verbatim}
\usepackage[round]{natbib}

\usepackage{hyperref}
\hypersetup{colorlinks=true,linkcolor=MidnightBlue,citecolor=MidnightBlue,urlcolor=MidnightBlue}

\newtheorem{theorem}{Theorem}

\newtheorem{proposition}{Proposition}
\newtheorem{lemma}{Lemma}
\newtheorem{fact}{Fact}

\theoremstyle{definition}
\newtheorem{definition}{Definition}
\newtheorem{example}{Example}
\newtheorem{question}{Question}

\DeclareMathOperator*{\argmin}{arg\,min}

\begin{document}

\title[Learnability with computable learners]{On characterizations of \\ learnability with computable learners}
\author[Sterkenburg]{Tom F.\ Sterkenburg}
\date{June 15, 2022.}
\thanks{This is the final version, as presented at the Conference on Learning Theory (COLT) 2022 and published in \emph{Proceedings of Machine Learning Research} 178: 3365--3379 (\href{https://proceedings.mlr.press/v178/sterkenburg22a.html}{link}). For helpful discussion and feedback thanks go to Matthias Caro, Gemma De Las Cuevas, Benedict Eastaugh, Peter Gr\"unwald, Wouter Koolen, Hannes Leitgeb, and the anonymous COLT referees. This research was supported by the Deutsche
Forschungsgemeinschaft (DFG, German Research Foundation)—Projektnummer 437206810, \emph{Die Epistemologie der Statistischen Lerntheorie}. Part of this research was done while I was visiting  the Machine Learning group of the CWI Amsterdam.}
\address{Munich Center for Mathematical Philosophy, LMU Munich}
\email{tom.sterkenburg@lmu.de}

\maketitle

\begin{abstract}%
We study computable PAC (CPAC) learning as introduced by \citet{aablu20alt}. First, we consider the main open question of finding characterizations of proper and improper CPAC learning. We give a characterization of a closely related notion of \emph{strong} CPAC learning, and provide a negative answer to the COLT open problem posed by \cite{aablu21colt} whether all decidably representable VC classes are improperly CPAC learnable. Second, we consider undecidability of (computable) PAC learnability. We give a simple general argument to exhibit such undecidability, and initiate a study of the arithmetical complexity of learnability. We briefly discuss the relation to the undecidability result of \citet{BHMSY19nmi}, that motivated the work of Agarwal et al.
\end{abstract}

\section{Introduction}
What changes in the theoretical analysis of learning algorithms when we impose a restriction to algorithms that are, in fact, algorithmic? %Among the more plausible-looking restrictions to impose on a theoretical framework for the analysis of automated learning methods we must surely count the restriction to learning algorithms that are, in fact, algorithmic.   
%Given that machine learning is concerned with the study of automated learning methods, it only makes sense to investigate a restriction of the theoretical analysis to learning algorithms that are, in fact, algorithms. 
This fundamental question led \citet{aablu20alt} to initiate a study of statistical learning theory with computable learners. The theory of probably approximately correct (PAC) learning, as presented by \citet{ShaBen14}, is founded on the Vapnik-Chervonenkis (VC) theory of uniform convergence \citeyearpar{VapChe71tpa}, that separates the statistical analysis of learning functions from computational considerations. On the other hand, PAC learning draws its name from Valiant's computational approach (\citeyear{Val84acm}; see \citealp{KeaVaz94}), that focuses on the efficiency (polynomial runtime) of learners. Agarwal et al.\ introduce a natural intermediate set-up, where it is (only) required for learners to be computable functions. They obtain several results about the ensuing notion of computable PAC (CPAC) learning and its relationship to unconstrained PAC learnability.

%The theory of probably approximately correct (PAC) learning draws its name from Valiant's computational approach, that focuses on the efficiency (polynomial runtime) of learners. However, as presented by \citet{ShaBen14}, PAC learning is essentially based on the Vapnik-Chervonenkis (VC) theory of uniform convergence \citeyearpar{VapChe71tpa}, where the statistical analysis of learning functions is separated from any computational considerations. Agarwal et al.\ present an intermediate set-up where we do not impose efficiency requirements, but we do take learners to be computable functions. 

The fundamental theorem of PAC learning \citep{BEHW89jacm} states that (under mild measurability conditions) a class of hypotheses is PAC learnable precisely if it satisfies the combinatorial property of finite VC dimension. Moreover, a class is PAC learnable precisely if the procedure of empirical risk minimization (ERM) PAC learns it. The main lesson that Agarwal et al.\ draw from their results is that the computability requirement ``disrupts the fundamental characterization of learnability by the finite VC-dimension of a class'' \citeyearpar[p.\ 59]{aablu20alt}. However, they leave as an open question what conditions \emph{do} characterize computable PAC learnability. As the two most important questions for future research,  they ask for characterizations of proper and of improper CPAC learnability. The latter motivates the open problem announced by \citet{aablu21colt}, whether there are decidably representable PAC learnable classes that are not even improperly CPAC learnable.

%Where the statistical Vapnik-Chervonenkis (VC) theory of uniform convergence \citeyearpar{VapChe71tpa} does not impose any computability constraints on learning functions, and Valiant's computational approach of probably approximately correct (PAC) learning \citeyearpar{Val84acm} focuses on the efficiency (polynomial runtime) of learning algorithms, Agarwal et al.\ present an intermediate set-up where we (only) require learners to be computable functions. They introduce the notion of computable PAC (CPAC) learning, and prove several results about the relationship between (variants of) standard PAC learnability and CPAC learnability.

In the first main part of this paper (Section \ref{sec:chars}), we make progress on these two questions. We introduce a notion of \emph{strong} CPAC (SCPAC) learnability, by adding a stipulation on the computability of the sample complexity. The motivation for this notion is that we can prove a natural characterization (that does preserve the classical characterization as neatly as possible), namely as the conjunction of finiteness of VC dimension and computability of ERM. In fact, the notions of CPAC and SCPAC learnability are so close that they may already be equivalent; we leave this as an open question.  Further, we solve the open problem of \citet{aablu21colt}. We confirm their conjecture that a particular decidably representable PAC learnable class is not even improperly CPAC learnable, implying that there is a nontrivial question of characterizing improper CPAC learnability.

An incentive for the work of Agarwal et al.\ was the result due to  \citet{BHMSY17arxiv,BHMSY19nmi} that learnability can be undecidable. Ben-David et al.\ introduce a general learning model of ``estimating the maximum'' (EMX),  and exhibit a particular EMX learnability problem that they prove to be independent of the ZFC axioms of set theory (provided ZFC is consistent). From this result they infer that ``there is no VC dimension-like parameter that generally characterizes learnability'' \citeyearpar[p.\ 44]{BHMSY19nmi}. Their analysis is that ``the source of the problem is in defining learnability as the existence of a learning function rather than the existence of a learning algorithm'' %: the separation of ``the statistical or information-theoretic issues from any computational considerations [\dots] has a high cost when it comes to more general types of learning.'' 
(ibid., p. 48). In the same vein, \citet[p.\ 48]{aablu20alt} write that ``[h]ad we required learners to be computable, there would have been a finite representation for each learner [\dots], ruling out independence of ZFC results of the type shown in \citet{BHMSY17arxiv,BHMSY19nmi}.'' 

In the second main part of this paper (Section \ref{sec:undec}), we turn to the undecidability of computable PAC learnability. %\citet{Car21arxiv} already examined the undecidability of (among other models) PAC learnability, but his results still only partially transfer to computable PAC learning. 
On the basis of Rice's Theorem, we offer a simple  argument to the effect that, for \emph{any} notion of learnability in the current computable framework, and a general approach to formulating decision problems of learnability (computable families of hypothesis classes), the resulting decision problem, if not trivial (either every class is learnable or every class is not), is unsolvable. We observe that the unsolvability of a learnability decision problem directly entails that the learnability of infinitely many hypothesis classes  is independent of the ZFC axioms (provided ZFC is arithmetically sound). Further, we initiate an investigation (similar to the work of \citealp{Ber14jsl,BBFGWY21aml} for algorithmic learning theory) into \emph{how} undecidable learnability problems are: that is, into their arithmetical complexity.  In particular, we use our characterization of SCPAC learnability to show that this decision problem is $\Sigma_3$-complete. Finally (in Section \ref{sec:concl}), we briefly discuss how our observations relate to the undecidability result of Ben-David et al.

\subsection*{Related work} 
We restrict attention to the framework of \citet{aablu20alt}, where the domain set is countable and hypotheses are total computable functions (see Section \ref{sec:prel}). \citet{AADFT21arxiv} present results about computable PAC learning within a more general framework of computable analysis, %\citep{Wei00}, 
where the domain is an arbitrary computable metric space. They also remark on the assumption of a computable sample complexity, the added ingredient in our notion of SCPAC learning. \citet{Cal15aml} already studied a computable setting where the domain is $ 2^\omega$ and hypotheses are $\Pi^0_1$ classes, % \citep{Cen99inc},
 and established the arithmetical complexity of PAC learnability (finiteness of VC dimension) of effective hypothesis classes within this setting. Calvert further notes the relation to earlier work on the computational complexity of calculating the VC dimension of finite hypothesis classes over finite domain \citep{LinManRiv91ic,Sch99jcss}. Schaefer, citing \citet{Weh90phd}, also gives the arithmetical complexity of PAC learnability within the computable setting we study here. \citet{Car21arxiv} recently showed the undecidability of (among other models) PAC learning, constructing instances of both ``Turing undecidability'' (unsolvability of decision problem) and ``G\"odel undecidability'' (independence of axiom system). % of learnability for PAC learning, two models of online learning, and a model of teacher-learner interactions. 
His constructions for the undecidability of PAC learning apply to the current computable setting, and indeed the relevant (families of) hypothesis classes are computable, but they only partly transfer to CPAC learnability (see Section \ref{ssec:undec} for more details). \citet{Ber14jsl,BBFGWY21aml} study the arithmetical complexity of learnability for the algorithmic learning theory paradigm of identification in the limit \citep{Gol67ic,JaiOshRoySha99}.

\section{Preliminaries}\label{sec:prel}

%\section{Undecidability in effective PAC learning}

\subsection{PAC learning}

Let $\mathcal{X}=\mathbb{N}$ the domain, and $\mathcal{Y}=\{0,1\}$ the label space. A hypothesis is a function $h: \mathcal{X} \rightarrow \mathcal{Y}$. %It will often be convenient to identify an hypothesis $h$ with its representation by the set $\{ x \in \mathcal{X}: h(x)=1 \} \subseteq \mathbb{N}$. An \emph{hypothesis class} $\mathcal{H}$ is a set of hypotheses, or with the previous representation of hypotheses as subsets of $\mathbb{N}$, a class $\mathcal{H} \subseteq \mathcal{P}(\mathbb{N})$. 
A \emph{sample} $S$ is a finite ordered sequence of input-label pairs, or formally, $S \in \mathcal{S} := \cup_{n\in \mathbb{N}} (\mathcal{X} \times \mathcal{Y})^n$. %(So samples are not \emph{sets} of input-label pairs: repetitions and even order matters.) 
To assess hypotheses, we use the 0/1 error function. Thus the \emph{error} of $h$ on sample $S$ is given by 
\begin{align*}
L_S(h) := \frac{|\{(x,y)\in S: h(x) \neq y \}|}{|S|},
\end{align*}
and the true error or  \emph{risk} of $h$ w.r.t.\ a distribution $\mathcal{D}$ over $\mathcal{X} \times \mathcal{Y}$ is 
\begin{align*}
 L_\mathcal{D}(h) :=\mathbb{P}_{(x,y) \sim \mathcal{D}}[h(x) \neq y].
%%L_\mathcal{D}(h) := \mathbb{E}_{(x,y) \sim \mathcal{D}}\left[L_{(x,y)}(h) \right].
\end{align*}

\begin{definition}[PAC learnability]
A hypothesis class $\mathcal{H}$ is PAC learnable if there exists a function $m_\mathcal{H}: (0,1)^2 \rightarrow \mathbb{N}$ and a learning function $A: \mathcal{S} \rightarrow \mathcal{H}$ such that for all $\epsilon,\delta \in (0,1)$, for all $m \geq m_\mathcal{H}(\epsilon,\delta)$ and any distribution $\mathcal{D}$ over $\mathcal{X} \times \mathcal{Y}$ we have
\begin{align}\label{eq:pac}
\mathrm{Prob}_{S \sim \mathcal{D}^m}\left[L_\mathcal{D}(A(S)) \leq \min_{h \in \mathcal{H}} (L_\mathcal{D}(h))+ \epsilon \right] \geq 1-\delta.
\end{align}
\end{definition}
We also call the above \emph{agnostic} PAC learning to distinguish it from the more specific case of \emph{realizable} PAC learning, where we make the assumption that there exists $h^* \in \mathcal{H}$ with $L_\mathcal{D}(h^*)=0$. We also call the above \emph{proper} PAC learning to distinguish it from the more general case of \emph{improper} PAC learning, where we do not assume that the range of the learning function $A$ is restricted to $\mathcal{H}$. That is, $A$ may also output hypotheses that are not in $\mathcal{H}$; but condition \eqref{eq:pac}, including the comparison to the best hypothesis \emph{in} $\mathcal{H}$, does not change.
\begin{definition}
Empirical risk minimization for hypothesis class $\mathcal{H}$, write \textsc{ERM}$_\mathcal{H}$, returns for each $S \in \mathcal{S}$ a hypothesis in $\argmin_{h \in \mathcal{H}}L_S(h)$.
\end{definition}
For hypothesis class $\mathcal{H}$ and $X= \{x_1,\dots,x_m\} \subset \mathcal{X}$, the \emph{restriction of $\mathcal{H}$ to $X$} is the class $\mathcal{H}_{|X}$ of functions $f: X \rightarrow \mathcal{Y}$ such that $f(x) = h(x)$ for some $h \in \mathcal{H}$ and all $x \in \mathcal{X}$.
%defined by %$\mathcal{H}_{|X}:=\{(h(x_1),\dots,h(x_m)): h \in \mathcal{H}\}$. 
%$\mathcal{H}_{|X}:=\{f : (\exists h \in \mathcal{H})( \forall x \in X)[ h(x)=f(x)] \}$.
We say that $\mathcal{H}$ \emph{shatters} finite $X \subset \mathcal{X}$ if the restriction of $\mathcal{H}$ to $X$ contains \emph{all} functions $f: X \rightarrow \mathcal{Y}$.%, that is, $|\mathcal{H}_{|X}|=2^{|X|}$.
\begin{definition}
The \emph{VC dimension} of hypothesis class $\mathcal{H}$, write VCdim$(\mathcal{H})$, is the maximal size of a set $X \subset \mathcal{X}$ that is shattered by $\mathcal{H}$. If $\mathcal{H}$ shatters sets of arbitarily large size, then VCdim$(\mathcal{H})=\infty$. %the VC dimension of $\mathcal{H}$ is infinite.
\end{definition}
\begin{theorem}[Fundamental theorem of PAC learning, \citealp{BEHW89jacm}]\label{theorem:funpac}
A hypothesis class $\mathcal{H}$ is PAC learnable if and only if ERM$_{\mathcal{H}}$ PAC learns $\mathcal{H}$ if and only if VCdim$(\mathcal{H})<\infty$.
%The following are equivalent:
%\begin{itemize}
%\item $\mathcal{H}$ is PAC learnable;
%\item ERM$_{\mathcal{H}}$ PAC learns $\mathcal{H}$;
%\item VCdim$(\mathcal{H})<\infty$.
%%\item $\mathcal{H}$ has the uniform convergence property.
%\end{itemize}
\end{theorem}

%The following weaker notion of PAC learnability is also standard.
%
%\begin{definition}[Nonuniform PAC learnability]
%Hypothesis class $\mathcal{H}$ is nonuniformly PAC learnable if there exists a function $m_\mathcal{H}_\mathrm{NUL} (0,1)^2 \times \mathcal{H} \rightarrow \mathbb{N}$ and a learning function $A: \mathcal{S} \rightarrow \mathcal{H}$ such that for all $\epsilon,\delta \in (0,1)$ and $h \in \mathcal{H}$, for all $m \geq m_\mathcal{H}_\mathrm{NUL}(\epsilon,\delta,h)$ and any distribution $\mathcal{D}$ over $\mathcal{X} \times \mathcal{Y}$ we have
%\begin{align}\label{eq:pac}
%\textrm{Prob}_{S \sim \mathcal{D}^m}\left[L_\mathcal{D}(A(S)) \leq \min_{h \in \mathcal{H}} (L_\mathcal{D}(h))+ \epsilon \right] \geq 1-\delta.
%\end{align}
%\end{definition}
%
%\begin{theorem}
%Hypothesis class $\mathcal{H}$ is nonuniformly PAC learnable iff there exists a sequence $(\mathcal{H}_n)_{n \in \mathbb{N}}$ of hypothesis classes such that $\mathcal{H} = \cup_{n} \mathcal{H}_n$ and $\mathrm{VCdim}(\mathcal{H}_n)<\infty$ for each $n$.
%\end{theorem}

\subsection{Computable PAC learning}

 We use the following computability-theoretic notation (see, e.g., \citealp{Soa16}). Let $\{ \phi_i \}_{i \in \mathbb{N}}$ be a standard enumeration of all partial computable (p.c.)\ functions. We write $\phi_i(x)\downarrow=y$ to denote that $\phi_i$ halts on input $x$ and returns $y$, while $\phi_i(x)\uparrow$ denotes that $\phi_i$ does not halt on $x$. We write $\phi_{i,s}(x)=y$ if $\phi_i$ outputs $y$ on input $x$ within $s$ computation steps; by convention, $i,x,y < s$. We similarly write $\phi_{i,s}(x)\downarrow$ if $\phi_{i}$ has halted and produced an output on $x$ by $s$ or $\phi_{i,s}(x)\uparrow$ if it has not.
 
%\subsubsection{Proper computable PAC learning}

In computable PAC (CPAC) learning, we work with \emph{computable} hypotheses, total computable functions $h: \mathcal{X} \rightarrow \mathcal{Y}$. Moreover, learners must be actual learning \emph{algorithms}, total computable functions from samples to computable hypotheses.

\begin{definition}[CPAC learnability, \citealp{aablu20alt}]\label{def:cpac}
A hypothesis class $\mathcal{H}$ is  CPAC learnable if there exists a \emph{total computable} $A: \mathcal{S} \rightarrow \mathcal{H}$ that PAC learns $\mathcal{H}$.
\end{definition}
We again also use the terms \emph{agnostic} and \emph{proper} to distinguish this notion from the more specific case of realizable CPAC learning and the more general case of improper CPAC learning.  

The following fact is an immediate consequence of Theorem \ref{theorem:funpac} and Definition \ref{def:cpac}. %We prove the converse later in the paper.
\begin{fact}\label{fact:efferm}
If VCdim$(\mathcal{H})<\infty$ and \textsc{ERM}$_{\mathcal{H}}$ is computably implementable, i.e., there is a total computable function that computes a version of \textsc{ERM}$_{\mathcal{H}}$, then $\mathcal{H}$ is CPAC learnable.  
\end{fact}
We further introduce a variant of CPAC learning, that we call \emph{strong} CPAC (or SCPAC) learning, where it is explicitly stipulated that the learning algorithm comes with a computable sample complexity function. We discuss the motivation for this notion in Section \ref{ssec:charprop}.
\begin{definition}[SCPAC learnability]\label{definition:scpac}
A hypothesis class $\mathcal{H}$ is SCPAC learnable if there exists a \emph{total computable}  $A: \mathcal{S} \rightarrow \mathcal{H}$ and a \emph{total computable} $m_\mathcal{H}: \mathbb{N}^2 \rightarrow \mathbb{N}$  such that for all $a,b \in \mathbb{N}$, for all $m \geq m_\mathcal{H}(a,b)$ and any distribution $\mathcal{D}$ over $\mathcal{X} \times \mathcal{Y}$, 
\begin{align}\label{eq:scpac}
\mathrm{Prob}_{S \sim \mathcal{D}^m}\left[L_\mathcal{D}(A(S)) > \min_{h \in \mathcal{H}} (L_\mathcal{D}(h))+ 1/a \right] < 1/b.
\end{align}
\end{definition}
The sufficient condition of Fact \ref{fact:efferm} is already sufficient for SCPAC learnability. 
\begin{proposition}\label{propo:efferms}
If VCdim$(\mathcal{H})<\infty$ and \textsc{ERM}$_{\mathcal{H}}$ is computably implementable, then $\mathcal{H}$ is SCPAC learnable.  
\end{proposition}
\begin{proof}
Given $\mathcal{H}$ with VCdim$(\mathcal{H})=d<\infty$, Sauer's lemma gives us a computable bound (depending only on the finite information $d$) on the sample complexity for the uniform convergence property of $\mathcal{H}$ \citep[Theorem 6.7]{ShaBen14}, which in turn gives us a computable bound on the sample complexity of \textsc{ERM}$_\mathcal{H}$ (ibid., Corollary 4.4). 
\end{proof}
\subsection{Computability of hypothesis classes}
We would also like to formulate a notion of effective computability of hypothesis classes, classes of computable hypotheses. Namely, a class of total computable functions can \emph{itself} fail to be computable, in the sense that there is no computable way of checking or even enumerating its elements.
\begin{example}[\citealp{aablu20alt}, Theorem 9]
Define for $i\in \mathbb{N}$ hypothesis $h_i$ by
\begin{align*}
h_i(x) =\begin{cases}
1 & \textrm{if } x=2i \textrm{ or }x=2i+1 \ \& \ \phi_i(i)\downarrow \\ 
0 & \textrm{otherwise}
\end{cases}
\end{align*}
and let hypothesis class $\mathcal{H}_\textrm{halt} := \{h_i\}_{i \in \mathbb{N}}$. While each \emph{individual} $h_i$ is  computable  (since given by finite information), the $h_i$ are not \emph{uniformly} computable in $i$ (or we could solve the Halting problem), meaning the members of $\mathcal{H}_\textrm{halt}$ cannot be computably enumerated. This underlies the fact that $\mathcal{H}_\textrm{halt}$ is not CPAC learnable, not even in the realizable case. Namely, by Fact  \ref{fact:efferm} it would suffice for CPAC learnability that \textsc{ERM}$_\mathcal{H_\mathrm{halt}}$ is computably implementable. For this,  in the realizable case, it would suffice that the elements of $\mathcal{H_\mathrm{halt}}$ can be enumerated \citep[Theorem 10]{aablu20alt}.
%
%. Namely, in the realizable case, one can then compute \textsc{ERM}$_\mathcal{H_\mathrm{halting}}(S)$ for any $S$ by enumerating elements until an $h$ is found with $L_S(h)=0$ \citep[theorem 10]{aablu20alt}.
%
%--
%
%This is the essence of theorem 9 of , that states the existence of a class of computable hypotheses of VC dimension 1 that is not even CPAC learnable in the realizable case. They define a class $\mathcal{H_\mathrm{halting}}=\{ h_i \}_i$ such that each $h_i$ encodes whether or not $\phi_i(\epsilon)\downarrow$ (this is in each instance finite information so \emph{each individual} $h_i$ is computable). By fact \ref{fact:efferm}, it suffices for CPAC learnability that \textsc{ERM}$_\mathcal{H_\mathrm{halting}}$ is effectively implementable. In the realizable case, it suffices for this that the elements of $\mathcal{H_\mathrm{halting}}$ can be effectively enumerated. Namely, in the realizable case, one can then compute \textsc{ERM}$_\mathcal{H_\mathrm{halting}}(S)$ for any $S$ by enumerating elements until an $h$ is found with $L_S(h)=0$ \citep[theorem 10]{aablu20alt}. However, the authors show that CPAC learnability of $\mathcal{H_\mathrm{halting}}$ would give us an effective procedure to solve the Halting problem. Hence this is impossible, which here essentially goes back to the impossibility of effectively enumerating $\mathcal{H_\mathrm{halting}}$.
\end{example}
%It makes sense to call a class of computable hypotheses \emph{effective} if it is at least enumerable in an effective manner. To make this more precise, we have to say a little bit more about the representation or encoding of computable hypotheses.
As a general approach to a notion of effective hypothesis classes, we always assume some encoding that computably corresponds the natural numbers (indices) to programs (Turing machines) for computing hypotheses, inducing some \emph{base class} $\hat{\mathcal{H}}$ of computable hypotheses. More precisely, we assume a computable \emph{decoding function} $C: \mathbb{N} \rightarrow \hat{\mathcal{H}}$, that gives a computable listing $\{ h_i \}_{i \in \mathbb{N}} = \hat{\mathcal{H}}$ by $h_i := C(i)$. Note that this base class $\hat{\mathcal{H}}$ must always be a strict subclass of the class $\mathcal{H}_\mathrm{comp}$ of \emph{all} computable hypotheses, because the computable hypotheses are the total computable (t.c.) functions, and by a standard diagonalization argument we cannot effectively enumerate or encode (programs for) all and only the t.c.\ functions.

Given such an encoding of a base class $\hat{\mathcal{H}}$, the available hypothesis classes $\mathcal{H} \subseteq \hat{\mathcal{H}}$ correspond to (and can be identified with) the subsets of $\mathbb{N}$. A computable subset of $\mathbb{N}$ then gives a computable class of (codes of) hypotheses in $\hat{\mathcal{H}}$, and a c.e.\ subset of $\mathbb{N}$ gives a computable enumeration of (codes of) a class of hypotheses in $\hat{\mathcal{H}}$. 
 %we can define a \emph{computable} hypothesis class as a computable $\mathcal{H} \subseteq \mathbb{N}$, and a c.e.\ or simply \emph{effective} hypothesis class as a c.e.\ $\mathcal{H} \subseteq \mathbb{N}$.  
 The former corresponds to the notion of \emph{decidably representable} (DR) hypothesis class of \citet{aablu20alt}, and the latter to their notion of \emph{recursively enumerably representable} (RER) hypothesis class. We will adopt this terminology. 
\begin{example}
Consider the base class $\mathcal{H}_\mathrm{fin}$ of all hypotheses \emph{with finite support}: the hypotheses $h$ with $x_0$ such that $h(x)=h(x')$ for all $x,x'>x_0$ \citep[Remark 7]{aablu20alt}. Each such $h$ is given by the finite information of its corresponding $x_0$, the list of labels for $x \leq x_0$, and the constant label for $x>x_0$; and we can clearly specify an encoding of all such hypotheses that gives a computable decoding function $C: \mathbb{N} \rightarrow \mathcal{H}_\mathrm{fin}$. Examples of DR subclasses---or choices of base classes in their own right---are the class $\mathcal{H}_\mathrm{ivl}$ of \emph{interval hypotheses} ($h$ with $x_0,x_1$ such that $h(x)=1$ iff $x_0 < x < x_1$) and the class $\mathcal{H}_\mathrm{thd}$ of \emph{threshold hypotheses} ($h \in \mathcal{H}_\mathrm{ivl}$ with $x_0=0$). An example of a non-DR RER subclass is the class of threshold hypotheses with $x_1$ such that $\phi_{x_1}(x_1)\downarrow$. 
\end{example}
 %The current notion of c.e.\ hypothesis class also corresponds, perhaps somewhat confusingly, to the notion of ``computable hypothesis class'' of \cite{Car21arxiv}. To make more sense of this, note that a c.e.\ hypothesis class $\mathcal{H}$ is \emph{uniformly computable} in that we can define a computable function $f$ such that $f(i,x) = h_i(x)$ for the $i$-th hypothesis in $\mathcal{H}$
Not every DR hypothesis class is CPAC learnable \citep[Theorem 11]{aablu20alt}, which means there are DR classes for which ERM is not computably implementable. On the other hand, a hypothesis class does not have to be RER to be SCPAC learnable.
\begin{example}\label{ex:scpacnotrer}
Take the class $\mathcal{H}_\mathrm{ivl}$  of interval hypotheses. This class has VC dimension 2 and \textsc{ERM}$_{\mathcal{H}_\mathrm{ivl}}$ is clearly computably implementable, so it is SCPAC learnable. Now extend this class with all threshold functions $h_a$ such that $\phi_a(a)\uparrow$. The extended class $\mathcal{H}'$ is no longer RER. However, $\mathrm{VCdim}(\mathcal{H}_\mathrm{ivl})=\mathrm{VCdim}(\mathcal{H}')$ and we have that for each $S \in \mathcal{S}$, $\min_{h \in \mathcal{H}_{\mathrm{ivl}}} L_S(h) = \min_{h \in \mathcal{H}'} L_S(h)$, so that the algorithm for \textsc{ERM}$_{\mathcal{H}_{\mathrm{ivl}}}$ also implements \textsc{ERM}$_{\mathcal{H}'}$. Thus $\mathcal{H}'$ is also SCPAC learnable. 
\end{example}

%
%representation so that it is at least decidable whether a given code represents a valid hypothesis --> restrict codes to valid hypotheses/give encoding of subclass of valid hypotheses. w.l.o.g. natural numbers. then hyp. classes are sets of natural numbers, effective classes are c.e. sets, comp.ble classes are computable sets. 
%  
%
%
%extent a notion of computability to hypothesis classes, sets of total computable functions. naively, for such a class we would have an algorithm that for any code of (a program for?) a computable hypothesis tells us whether this hypothesis is in the class or not.
%
%[decision algorithm that for every index of nontotal c.e. function returns NO---doesn't exist...]
%
% problem: we cannot effectively represent/encode *all* total computable functions. 

%\subsubsection{Nonuniform learning}

\section{Towards characterizations of computable learnability}\label{sec:chars}
%We now turn to the problem of characterizing (strong) CPAC learning. We first look at the proper case (section \ref{ssec:charprop}) and then at the improper case (section \ref{ssec:charimpr}).
\subsection{Proper (S)CPAC learnability}\label{ssec:charprop}
We saw that a hypothesis class is (S)CPAC learnable if it has finite VC dimension and \textsc{ERM} is computably implementable. For SCPAC learnability, this condition pair is also necessary.
\begin{theorem}\label{theorem:cpacefferm}
A hypothesis class $\mathcal{H}$ is properly SCPAC learnable if and only if VCdim$(\mathcal{H})<\infty$ and there exists an algorithm that implements \textsc{ERM}$_\mathcal{H}$.
\end{theorem}
\begin{proof}
It remains to show the left-to-right direction. So suppose $\mathcal{H}$ is SCPAC learnable. Then $\mathcal{H}$ is PAC learnable, so VCdim$(\mathcal{H})<\infty$; and there are computable learning function $A$ and computable sample complexity function $m_\mathcal{H}$ such that for all $a,b \in \mathbb{N}$, for all $m \geq m_\mathcal{H}(a,b)$ and any distribution $\mathcal{D}$ over $\mathcal{X} \times \mathcal{Y}$ we have \eqref{eq:scpac}.
%\begin{align}\label{eq:suppscpaclearn}
%\textrm{Prob}_{S \sim \mathcal{D}^m}\left[L_\mathcal{D}(A(S)) > \min_{h \in \mathcal{H}} (L_\mathcal{D}(h))+ \epsilon \right] < \delta.
%\end{align} 
%drawing a sample $S^m \sim \mathcal{D}$ with $m \geq m(\epsilon,\delta)$ we have with $\mathcal{D}$-probability at least $1-\delta$ that $L_{\mathcal{D}}(A(S^m))\leq\min_{h \in \mathcal{H}}(L_{\mathcal{D}}(h))+\epsilon$. 
Using $A$, we can computably implement \textsc{ERM}$_\mathcal{H}$ as follows. 

For given training sample $S = ((x_1,y_1),\dots,(x_n,y_n))$, %\in \mathcal{S}^*$,
 define distribution $\mathcal{D}_S$ by $\mathcal{D}_S((x_i,y_i))=1/n$ for each $(x_i,y_i) \in S$ (in case of repetitions in $S$, we simply add up the probabilities). Choose $a>n$ and any $b$, and compute $m = m_\mathcal{H}(a,b)$. Let $\mathcal{S}_{\mathcal{D}_S}^m$ be the set of all possible length-$m$ samples that can be generated from $\mathcal{D}_S$. % samples of length $m$ that contain only (but not necessarily all) elements $(x,y) \in S$. 
By running $A$ on all these sequences, we can computably pick some $\hat{S} \in \arg\min_{S' \in \mathcal{S}_{\mathcal{D}_S}^m} L_S(A(S'))$. The claim is that $\hat{h}=A(\hat{S})$ is also in $\arg\min_{h \in \mathcal{H}}L_S(h)$. Namely, if not, then for all $S' \in \mathcal{S}_{\mathcal{D}_S}^m$ we would have $L_S(A(S'))> \min_{h \in \mathcal{H}}L_S(h)$. Specifically, each $A(S')$ would make at least one more mistake on $S$ than the $h \in \arg \min_{h \in \mathcal{H}}L_S(h)$, which by definition of $\mathcal{D}_S$ implies $L_{\mathcal{D}_S}(A(S'))\geq \min_{h \in \mathcal{H}}L_{\mathcal{D}_S}(h)+1/n$. But that implies that with certainty ($\mathcal{D}^m_S$-probability 1) we would sample $S' \sim \mathcal{D}_S$ of length $m$ with $L_{\mathcal{D}_S}(A(S')) > \min_{h \in \mathcal{H}} L_{\mathcal{D}_S}(h) + 1/a$, contradicting \eqref{eq:scpac}. 
\end{proof}
We do not know whether CPAC learnability is not already equivalent to SCPAC learnability. If it is not, then Theorem \ref{theorem:cpacefferm}, which constitutes an effective version of the original equivalence between PAC learnability and PAC learnability by \textsc{ERM}, gives reason for thinking that SCPAC learnability is a natural notion. Moreover, the above proof suggests that an $\mathcal{H}$ that is CPAC but not SCPAC learnable has extreme properties. In particular, it can only be learnable by an algorithm $A$ for which we cannot compute an upper bound on any corresponding sample complexity function $g_{b}(a)=m(a,b)$ for fixed $b$. That is to say, the sample complexity must grow faster in $a$ than any computable function.
\begin{question}\label{q:cpac=scpac}
Does there exist a hypothesis class that is properly CPAC learnable but not properly SCPAC learnable?
\end{question}
In any case, both the negative and the positive results on CPAC learning in \citep{aablu20alt} also go through for SCPAC learning: the first (Theorems 9 and 11) because the latter is stronger, the second (Theorems 10, 13, and 15; Corollary 14) because they rely on showing the computable implementability of \textsc{ERM}, which  already gives SCPAC learnability.

\subsection{Improper (S)CPAC learnability}\label{ssec:charimpr}

We now turn to the improper case. To be clear,  we use the qualifier ``improper'' as a generalization of ``proper.''  We will use the qualifier ``strictly improper'' to mean ``improper but not proper.'' 
The following fact is immediate from the definitions.
\begin{fact}\label{fact:primpr}
%If $\mathcal{H}$ is CPAC learnable and $\mathcal{H}' \subseteq \mathcal{H}$, then \textsc{ERM}$_\mathcal{H}$ also (improperly) CPAC learns $\mathcal{H}'$.
If $\mathcal{H}$ is improperly CPAC learnable and $\mathcal{H}' \subseteq \mathcal{H}$ then $\mathcal{H}'$ is improperly CPAC learnable. The same holds for improper SCPAC learnability.
\end{fact}
%\begin{fact}\label{fact:primpr}
%%If $\mathcal{H}$ is CPAC learnable and $\mathcal{H}' \subseteq \mathcal{H}$, then \textsc{ERM}$_\mathcal{H}$ also (improperly) CPAC learns $\mathcal{H}'$.
%If $\mathcal{H}$ has a superclass $\mathcal{H}' \supseteq \mathcal{H}$ that is improperly CPAC learnable  then $\mathcal{H}$ is improperly CPAC learnable. The same holds for improper SCPAC learnability.
%\end{fact}
\citet[Theorem 9, Theorem 11]{aablu20alt} exhibit two classes $\mathcal{H}_\textrm{halting}$ and $\mathcal{H}_\textrm{LT}$ that are not properly CPAC learnable, yet that \emph{are} improperly (so \emph{strictly} improperly) CPAC (indeed SCPAC) learnable (ibid., p.\ 59). Intuitively, the reason is that the incomputable information encoded in these classes can be ``blotted out'' by adding more hypotheses. This is easy if (as in the case of $\mathcal{H}_\textrm{halting}$ and $\mathcal{H}_\textrm{LT}$) a class  only contains, for some constant $b$, hypotheses (seen as sets of positively labeled instances) of size bounded by $b$. Then the obvious SCPAC learnability of the superclass of \emph{all} such $b$-bounded-size hypotheses means by Fact \ref{fact:primpr} that the original class is improperly SCPAC learnable.\footnote{A class need not have this boundedness property for similar reasoning to go through, as shown by an example of one of the referees. For any $b$-bounded-size $\mathcal{H} = \{h_i\}_i$ define $\mathcal{H}' = \{h_i'\}_i$ by $h_i'(x)=h_i(x)/2$ if $x$ is even and $h_i'(x)=1$ otherwise,
yielding a class of infinite hypotheses that is nevertheless extendable to a properly CPAC learnable class. An interesting further question is to find a more concrete characterization of such extendability.}
%In light of this, a 
%\footnote{A 
%	strictly improperly CPAC learnable hypothesis class need not have this boundedness property, as shown by an example of an anonymous referee. For any strictly improperly CPAC learnable $\mathcal{H} = \{h_i\}_i$ define $\mathcal{H}' = \{h_i'\}_i$ by $h_i'(x)=h_i(x)/2$ if $x$ is even and $h_i'(x)=1$ otherwise,
%%\begin{align*}
%%h_i'(x) =
%%\begin{cases}
%% h_i(x/2) & \textrm{if } x \textrm{ is even} \\
%% 1 & \textrm{otherwise},
%%\end{cases}
%%\end{align*}
%yielding a class of infinite hypotheses that is nevertheless improperly SCPAC learnable. An interesting further question is to find a more concrete characterization of extendability to a properly CPAC learnable class.}
%In light of this, a 

%This raises, first, the question whether we can give a more concrete characterization of when a class can be extended to a properly CPAC learnable class. 

In general, by Fact \ref{fact:primpr}, extendability to a proper (S)CPAC learnable class is sufficient  for improper (S)CPAC learnability; the next question, towards an actual characterization, is whether it is actually a \emph{necessary} condition \citep[Conjecture 23]{aablu20alt}. %We do not settle these questions here. But we will prove 
But a preceding question is whether, at least for RER hypothesis classes, there is not already a more trivial characterization: \emph{every} RER class with finite VC dimension is improperly (S)CPAC learnable. We show here that this is not the case. 
%We show here that not every RER hypothesis class with finite VC dimension is already improperly CPAC learnable, so that the latter question (restricted to RER classes) is not rendered obsolete by a  characterization by mere finiteness of VC dimension. %such a characterization is not trivial in the sense that \emph{every} RER hypothesis class with finite VC dimension is already improperly CPAC learnable. 
For this purpose we take the hypothesis class $\mathcal{H}_\textrm{init}$ defined by \citet[p.\ 4641]{aablu21colt}, which they already conjecture is not even improperly CPAC learnable (ibid., Conjecture 9). We slightly reformulate their definition. 
Let, for each $s \in \mathbb{N}$, computable hypothesis $h_s$ be defined by
\begin{align*}
h_s(x) =
\begin{cases}
 1 & \textrm{if } \phi_{x,s}(x)\downarrow  \\ %latter implies x<s
 0 & \textrm{otherwise},
\end{cases}
\end{align*}
and let $\mathcal{H}_\textrm{init} := \{ h_s \}_{s \in \mathbb{N}}$. This class is in fact DR  and has VC dimension 1. %(for any $x_1,x_2$, if there is $h_t$ with $h_t(x_1)=0,h_t(x_2)=1$ then there cannot be $h_s$ with $h_s(x_1)=1,h_s(x_2)=0$) 
%and is clearly not (co-)constant-bounded. We show that this class is indeed not even improperly CPAC learnable. 
First we need a lemma.%, that is essentially an application of the computable No-free-lunch Theorem \citep[Lemma 19]{aablu20alt} 

\begin{lemma}\label{lmm:cpacoutg}
If $\mathcal{H}$ is improperly CPAC learnable, then for %each $m>2 \cdot \textrm{VCdim}(\mathcal{H})$ and
sufficiently large $n$, we can computably find for any $X=\{x_1,\dots,x_n\}  \subset \mathcal{X}$ of size $n$ a function $g: \{ x_1,\dots,x_n\} \rightarrow \{0,1\}$ with $g \notin \mathcal{H}_{|X}$.
\end{lemma}
\begin{proof}
Suppose there exists an algorithm $A$ that improperly learns $\mathcal{H}$. Picking some $a > 8$ and $b > 7$, that means that there is sufficiently large $m=m(\epsilon,\delta)$ such that for any $\mathcal{D}$ over $\mathcal{X} \times \{0,1\}$  
\begin{align*}
\textrm{Prob}_{S \sim \mathcal{D}^m}\left[L_\mathcal{D}(A(S)) \geq \min_{h \in \mathcal{H}} (L_\mathcal{D}(h))+ 1/8 \right] < 1/7.
\end{align*}
But by the computable No-Free-Lunch Theorem \citep[Lemma 19]{aablu20alt}, for any $X=\{x_1,\dots,x_{n}\}  \subset \mathcal{X}$ of size $n=2m$ we can computably find a function $g: \{ x_1,\dots,x_{n}\} \rightarrow \{0,1\}$  such that for distribution $\hat{\mathcal{D}}$ uniform over $\{(x_1,g(x_1)),\dots,(x_n,g(x_n)) \}$ we have
\begin{align*}
\textrm{Prob}_{S \sim \hat{\mathcal{D}}^m}\left[L_{\hat{\mathcal{D}}}(A(S)) \geq 1/8 \right] \geq 1/7.
\end{align*}
%with $\mathcal{D}$ the uniform distribution over $(x^m,f(x^m))$. 
This implies that $g \notin \mathcal{H}_{|X}$, for else $\min_{h \in \mathcal{H}} (L_{\hat{\mathcal{D}}}(h))=0$ and we would have a contradiction. 
%This follows by an application of the computable No-Free-Lunch Theorem \citep[Lemma 19]{aablu20alt}, see Appendix \ref{sec:app}  for the details.
\end{proof}
\begin{theorem}
The class $\mathcal{H}_\mathrm{init}$ is not improperly CPAC learnable.
\end{theorem}
\begin{proof}
Suppose it is. % Then there is a class $\mathcal{H}' \supset \mathcal{H}$ that is proper CPAC learnable. That is to say, $\mathcal{H}'$ has finite VC dimension, say $m$, and the set $B_{\mathcal{H}'}$ is computable. 
Then by Lemma \ref{lmm:cpacoutg} there exists, for some sufficiently large $n$ of our choice, an algorithm $B$ that for any $n$ input elements $x_1,\dots,x_n$ proceeds as follows. If $x_i \neq x_j$ for all distinct $i,j \leq n$, then $B$ returns a function $g: \{ x_1,\dots,x_n\} \rightarrow \{0,1\}$  such that $g \notin \mathcal{H}_{\textrm{init}|\{x_1,\dots,x_n\}}$. If not, then $B$ returns some default function on $\{ x_1,\dots,x_n\}$, say the constant-0 function. 

We define, for each $i \leq n$, a total computable $n$-place $f_i$ such that
%\begin{align}
%W_{f_i(x_1,\dots,x_m)} = \begin{cases}
%\{ x_i \} & \textrm{if } \Phi(\{x_1,...,x_m\})(x_i)=1  \\
%\emptyset & \textrm{otherwise}
%\end{cases}
%\end{align}
%\textrm{if all } x_i \textrm{ are distinct and }
%y_i = 1 \textrm{ for first } y^m \textrm{ with } \{(x_1,y_1),\dots,(x_m,y_m) \} \notin B_{\mathcal{H'}}
\begin{align*}
\phi_{f_i(x_1,\dots,x_n)}(z) = \begin{cases}
 i & \textrm{if } z = 0 \\
 0 & \textrm{if } z = x_i >0 \ \& \ B(x_1,\dots,x_n)(x_i)=1  \\
\uparrow & \textrm{otherwise}.
\end{cases}
\end{align*}
%with the $b_i$ arbitrary but distinct.
Now by the $n$-fold Recursion Theorem \citep[p.\ 117]{Smu93} there are $c_1, \dots, c_n>0$ such that for each $i\leq n$,
\begin{align*}
\phi_{c_i} = \phi_{f_i(c_1, \dots, c_n)}.
\end{align*}
Moreover, these $c_1, \dots, c_n$ must be distinct, else $\phi_{f_i(c_1, \dots, c_n)}=\phi_{f_j(c_1, \dots, c_n)}$ for some $i\neq j$, which is excluded by the fact that they have distinct range (for each $i$ only $\phi_{f_i(c_1, \dots, c_n)}$ has $i$ in its range). 
But then for function $g=B(\{c_1, \dots, c_n \})$ we have for each $c_i$ that $\phi_{c_i}(c_i)\downarrow$ iff $\phi_{f_i(c_1, \dots, c_n)}(c_i)\downarrow$ iff $g(c_i)=1$. This means there exists a large enough $s $ such that for each $i\leq n$, $\phi_{c_i,s}(c_i) \downarrow$ iff $g(c_i)=1$, which implies by definition of $\mathcal{H}_\mathrm{init}$ that $g \in \mathcal{H}_{\mathrm{init}|\{c_1,\dots,c_n\}}$, contrary to specification of $B$.
%for $((\hat{a}^m,y^m)) \notin \mathcal{H}_{|X}$
%the first $\hat{S}=\{(\hat{a}_1,y_1),\dots,(\hat{a}_m,y_m) \} \notin B_{\mathcal{H'}} $. 
%But this must mean that there is $h \in \mathcal{H}$ such that $L_{\hat{S}}(h)=0$, in contradiction with $\hat{S} \notin B_{\mathcal{H'}} \supset B_\mathcal{H}$.
\end{proof}
%The next question is whether we can strengthen Proposition \ref{propo:cocbcpac} into a characterization of improper CPAC learnability. Is it the case that $\mathcal{H}$ is improperly CPAC learnable iff $\mathcal{H}$ is properly CPAC learnable \emph{or} $\mathcal{H}$ is (co-)constant-bounded, or can we give a counterexample?

It follows with Fact \ref{fact:primpr} that there are RER classes with finite VC dimension that cannot be extended to properly (S)CPAC learnable classes. So the latter extendability property is in this sense nontrivial; the question remains whether it actually characterizes improper (S)CPAC learnability.

%\begin{question}\label{q:cocb=imprcpac}
%Does there exist a hypothesis class $\mathcal{H}$ that is not properly CPAC learnable and not (co-)constant-bounded, yet is improperly CPAC learnable?
%\end{question}

\begin{question}[\citealp{aablu20alt}, Conjecture 23]\label{q:cocb=imprcpac}
Does there exist a (RER) class $\mathcal{H}$ that is not extendable to a properly (S)CPAC learnable class, yet that is improperly (S)CPAC learnable?
\end{question}

\section{Undecidability and complexity of learnability}\label{sec:undec}

\subsection{Undecidability}\label{ssec:undec}
There are two kinds of undecidability, that are related but not the same (see, e.g., \citealp[p.\ 211]{Poo14incol}; \citealp[pp.\ 251--52]{Ham20}; \citealp{Car21arxiv}).

\begin{enumerate}
\item \emph{Independence of a statement from an axiom system.}  A statement $Y$ is independent of (or undecidable in) axiom system $\mathcal{A}$ if neither $Y$ nor its negation can be derived from these axioms using the rules of logic. That is, neither $ \mathcal{A} \vdash Y$ nor $ \mathcal{A} \vdash \neg Y$. An example is the independence of the continuum hypothesis from the ZFC axioms of set theory. %Another example is the proof of \citet{BHMSY19nmi} of the independence of EMX learnability from ZFC.
\item \emph{Unsolvability of a decision problem.} A decision problem, i.e., a family $\{ Q_i \}_{i \in \mathbb{N}}$ of problems with YES/NO answers, is unsolvable (or undecidable) if there is no decision algorithm that on each input $i \in \mathbb{N}$ returns the correct answer to $Q_i$. The standard example is the unsolvability of the Halting problem, that asks for each $i \in \mathbb{N}$ whether $\phi_i(i)\downarrow$.
\end{enumerate}
Let ``learnable'' in this section stand for \emph{any} specific notion of learnability. We first consider the undecidability of learnability in sense (2), or the unsolvability  of a \emph{learnability decision problem}.  

To a first approximation, a learnability decision question asks: does there exist a decision algorithm that for every given hypothesis class  returns YES if the class is learnable and returns NO if it is not? To make this question meaningful at all, we must presuppose some family $\mathsf{H}$ of hypothesis classes such that each $\mathcal{H} \in \mathsf{H}$ can actually be presented as input to a candidate decision algorithm. %That is, we must have some effective encoding, this time of hypothesis \emph{classes}, so that we can ask the algorithm about hypothesis class $\mathcal{H}_j$ by presenting it input $j$. %not necessarily possible even if we restrict ourselves to \textsc{RER} (or DR) hypothesis classes.
%that induces an effective enumeration $\{ \mathcal{H}_i \}_{i \in \mathbb{N}}$, so that we can ask whether 
%be able to effectively encode this family such that each member can be presented in the form of finite information (say, again, a natural number) to the candidate decision algorithm.
\begin{example}
%Consider the family $\mathsf{H}^\mathrm{all}$ of \emph{all} hypothesis classes of computable hypotheses. 
It is impossible to effectively encode the family $\mathsf{H}^\mathrm{all}$ of \emph{all} hypothesis classes of computable hypotheses. A learnability problem for $\mathsf{H}^\mathrm{all}$ is therefore trivially undecidable: there exists no decision algorithm, because there cannot even exist an algorithm to query on each $\mathcal{H} \in \mathsf{H}^\mathrm{all}$.
\end{example}
%Thus in order to formulate a meaningful decision problem for a family $\mathsf{H}$, the family must itself be encodable. 
Let a \emph{computable} family $\mathsf{H} = \{ \mathcal{H}_j\}_{j \in \mathbb{N}}$ of hypothesis classes be such that there is a computable procedure that for each given $j \in \mathbb{N}$ retrieves an effective representation of $\mathcal{H}_j$; at the least, it uniformly retrieves an instruction for enumerating the elements of $\mathcal{H}_j$ (so the hypothesis classes of a computable family are all RER).
%uniform computable procedure to reconstruct i-th hypothesis class. at the least, to enumerate elements (so in particular an RER class). 
%Overloading terminology, let an RER family $\mathsf{H}$ of hypothesis classes be such that there exists computable function $D: \mathbb{N}^2 \rightarrow \hat{\mathcal{H}}$, for some base class $\hat{\mathcal{H}}$, such that $\mathsf{H} = \{ \mathcal{H}_j \}_j$ with $\mathcal{H}_j = \{h_{j,i}\}_{i \in \mathbb{N}}$ for $h_{j,i} := C(j,i)$. (In particular, each $\mathcal{H}_j$ is an RER hypothesis class.) 
For any computable family $\{ \mathcal{H}_j\}_{j \in \mathbb{N}}$ we can clearly state a corresponding \emph{decidability of learnability question}: does there exists an algorithm that for each input $j$ returns YES if $\mathcal{H}_j$ is learnable and NO otherwise? 

We describe a general way of constructing computable families of hypothesis classes, and show that for each family constructed in this way, the decision problem, if not trivial, is undecidable. Pick any base class $\hat{\mathcal{H}} \subset \mathcal{H}_\mathrm{comp}$ of computable hypotheses that we can code onto the natural numbers. %Via this encoding, the c.e.\ sets of natural numbers correspond to the RER hypothesis classes $\mathcal{H}\subseteq \hat{\mathcal{H}}$. Moreover, 
The uniformly c.e.\ family $\{W_i \}_{i \in \mathbb{N}}$ of all c.e.\ subsets of $\mathbb{N}$, or equivalently the family $\{ \phi_i \}_{i \in \mathbb{N}}$ of all p.c.\ functions, picks out the computable family $\mathsf{H}=\{ \mathcal{H}_i\}_{i \in \mathbb{N}}$ of all RER hypothesis classes $\mathcal{H}_i \subseteq \hat{\mathcal{H}}$. 
We call such a family $\mathsf{H}$  a \emph{maximal} computable family of hypothesis classes. Importantly, such a maximal family $\mathsf{H}=\{ \mathcal{H}_i\}_{i \in \mathbb{N}}$ has the property that if $\phi_i = \phi_j$ then also $\mathcal{H}_i = \mathcal{H}_j$. 

Now the answer to our question is \emph{yes}, for any computable family that either only contains learnable or only contains unlearnable hypotheses classes. For such a family that is \emph{trivial for learnability}, either the constant YES algorithm or the constant NO algorithm is a decision algorithm. 
\begin{example}
The maximal computable family constructed from the base class $\mathcal{H}_\mathrm{ivl}$ of interval hypothesis is a trivial family for PAC learnability: already the base class has finite VC dimension. This family is also trivial for improper (S)CPAC learning (as the base class is SCPAC learnable, Example \ref{ex:scpacnotrer}). However, the family is nontrivial for proper (S)CPAC learning: there exist RER classes of interval hypotheses that are not CPAC learnable \citep[Theorem 11]{aablu20alt}.
%In contrast, the maximal family constructed from the base class $\mathcal{H}_\mathrm{fin}$ of finitely supported hypotheses is nontrivial for each of the notions of PAC learnability, CPAC learnability, and SCPAC learnability.
\end{example}
But as soon as a maximal computable family is nontrivial for learnability, the answer is \emph{no}.
%More interesting is the decision question for effective families  that contain both learnable and unlearnable hypothesis classes. For instance, the maximal family stemming from the base class of all finitely supported functions is non-trivial for PAC learnability. 
\begin{proposition}\label{propo:undecunsolv}
For any particular notion of learnability, and any maximal computable family $\mathsf{H}$ of hypotheses classes that is nontrivial for this learnability, the learnability problem is unsolvable. %The same statement holds for learnability replaced by (improper) CPAC-learnability.
\end{proposition}
\begin{proof}
%The proof relies on the correspondence between the member of a maximal family and all c.e.\ sets. 
By the correspondence between the members of $\mathsf{H}$ and all p.c.\ functions, this follows directly from Rice's Theorem (see \citealp[p.\ 16]{Soa16}) that every \emph{nontrivial index set} is incomputable. An index set $I \subseteq \mathbb{N}$ is a set of indices of p.c.\ functions closed under extensional equivalence,
\begin{equation*}
i \in I \ \& \ \phi_i = \phi_j \Longrightarrow j \in I,
\end{equation*}
and nontrivial if neither $I = \emptyset$ nor $I = \mathbb{N}$. Now for any maximal computable family $\mathsf{H} = \{ \mathcal{H}_i \}_{i \in \mathbb{N}}$ of hypothesis classes, we have that if $\mathcal{H}_i$ is learnable and $\phi_i = \phi_j$, then $\mathcal{H}_i = \mathcal{H}_j$ and $\mathcal{H}_j$ is learnable, too; so that the set $I_{L(\mathsf{H})} = \{ i \in \mathbb{N}: \mathcal{H}_i \textrm{ learnable}\}$ is an index set, that is non-trivial if $\mathsf{H}$ is. But then Rice's Theorem says that $I_{L(\mathsf{H})}$ is incomputable, which just means that there can be no decision algorithm that for every $i$ returns YES if $i \in I_{L(\mathsf{H})}$ and NO otherwise.
\end{proof}
Undecidability is not limited to maximal computable families as constructed above. 
\begin{example}[\citealp{Car21arxiv}, Section 2.3]
Caro constructs a computable family $\mathsf{H}_\mathrm{halt}=\{ \mathcal{H}_{M_j} \}_{j \in \mathbb{N}}$ uniformly from the class $\{ M_j \}_{j \in \mathbb{N}}$ of Turing machines (i.e, the class $\{ \phi_j \}_{j \in \mathbb{N}}$ of p.c.\ functions), and proves undecidability of the PAC learnability problem for $\mathsf{H}_\mathrm{halt}$. This family also has the property that $\phi_i = \phi_j$ implies $\mathcal{H}_{M_i} = \mathcal{H}_{M_j}$, so that the previous reasoning by Rice's Theorem actually applies here too. Caro's own proof is a direct derivation of the undecidability of finiteness of VC dimension for $\mathsf{H}_\mathrm{halt}$, which entails undecidability of PAC learnability and also (as noted by \citealp[Section 5]{Car21arxiv}) of realizable CPAC learnability, as both are characterized by finite VC dimension (for RER classes). In fact, by Theorem \ref{theorem:cpacefferm}, finite VC dimension here already characterizes (agnostic) SCPAC learnability, because one can verify that all classes in $\mathsf{H}_\mathrm{halt}$ admit of a computable implementation of ERM. Still, the advantage of the generality of the reasoning by Rice's Theorem is that it directly gives us undecidabilility for \emph{any} learnability notion that $\mathsf{H}_\mathrm{halt}$ is nontrivial for.
\end{example}
Caro also already showed undecidability of PAC learning in sense (1).  
\begin{example}[\citealp{Car21arxiv}, Section 2.2]\label{ex:caroindzfc}
Caro presents a construction, for any sufficiently expressive formal system $F$, of an RER hypothesis class $\mathcal{H}_F$ such that $\mathcal{H}_F$ has finite VC dimension if and only if $F$ is consistent. Since, by G\"odel's second incompleteness theorem, $F$ (provided it is consistent) does not decide its own consistency, this yields, for any $F$, that $F$ does not decide the learnability of $\mathcal{H}_F$. In particular, ZFC (provided it is consistent) does not decide the learnability of $\mathcal{H}_{\textrm{ZFC}}$.
\end{example}
As \citet[Remark 2.24]{Car21arxiv} also notes, there is a way of directly deriving undecidability in sense (1) from undecidability in sense (2); so in particular from Proposition \ref{propo:undecunsolv}. We follow the reasoning outlined by \citet[pp.\ 212--13]{Poo14incol}. %We will derive independence of learnability from ZFC, under the provision that ZFC is not only consistent, but also  
\begin{proposition}\label{propo:undecindep}
Given any particular notion of learnability that we can arithmetically characterize (which includes PAC learnability and SCPAC learnability, see Section \ref{ssec:compl}). For any computable family $\mathsf{H}=\{ \mathcal{H}_i\}_{i \in \mathbb{N}}$ of hypothesis classes such that the learnability decision problem is unsolvable (in particular, any maximal computable family for which this learnability is nontrivial), the learnability of infinitely many $\mathcal{H}_i$ is independent of ZFC (provided ZFC is arithmetically sound).
\end{proposition}
\begin{proof}
Using the presupposed characterization  of the relevant notion of learnability, we can write a computable procedure that for each $i$ returns a statement $Y_i$ of first-order arithmetic 
that expresses that 
% such that $Y_i$ is true if and only if 
$\mathcal{H}_i$ is learnable. (For instance, for PAC learnability, the algorithm produces the statement \eqref{eq:aritmpac} in Section \ref{ssec:compl} below, uniformly plugging in arithmetical representations of the relevant ``atomic'' statements about computable objects, like $[h(x) \neq y_i]$.) If ZFC is  arithmetically sound, it only proves such statements (suitably recast in the language of set theory) that are in fact true. Thus we have a computable procedure that for each $i$ returns a statement $Y_i$ such that
\begin{itemize}
\item if ZFC $\vdash Y_i$ then $\mathcal{H}_i$ is learnable;
\item if ZFC $\vdash \neg  Y_i$ then $\mathcal{H}_i$ is not learnable.
\end{itemize}
But this gives us a decision procedure for learnability for $\mathsf{H}$ (for each $i$ enumerate theorems of ZFC until we find either $Y_i$ or $\neg Y_i$), \emph{unless} some (indeed infinitely many) $Y_i$ are independent of ZFC.
\end{proof}
\subsection{Arithmetical complexity}\label{ssec:compl}
The general proof by Rice's Theorem of undecidability of learnability does not use any specific properties of the notion(s) of learnability. The mathematical structure of learnability does come into play when we ask the natural next question, namely \emph{how} undecidable learnability is. Specifically,  what is the \emph{arithmetical complexity} of the relevant index set (see \citealp{Soa16})?

%After establishing that a problem is undecidable, the natural next question is \emph{how} undecidable (see \citealp{Soa16}). Specifically,  what is the arithmetical complexity of the relevant index set?

%We now take a look at \emph{how} unsolvable a learnability decision problem can be. Specifically, we look at its arithmetical complexity.
%\subsection{PAC learnability of RER families}

We start with standard PAC learnability, characterized by finiteness of VC dimension.  
We can spell out the property
$\textrm{VCdim}(\mathcal{H}) < d$ as %``no size-$d$ subset of $\mathcal{X}$ is scattered by $\mathcal{H}$,'' or 
%``for all size-$d$ subsets of $\mathcal{X}$, there exists a labeling that does not correspond to any hypotheses in $\mathcal{H}$,''  or %more formally 
\begin{align}\label{eq:aritmpac}
%\textrm{VCdim}(\mathcal{H}) \geq d &\Longleftrightarrow 
(\forall \ \textrm{distinct } x_1, \dots, x_d \in \mathcal{X}) (\exists y_1, \dots, y_d \in \{0,1\}) (\forall h \in \mathcal{H})(\exists i \leq d) \left[ h(x_i)\neq y_i \right].
%&\Longleftrightarrow \exists S \subseteq \mathcal{X} \left( |S| = %d \ \& \ |\mathcal{H}_S| = 2^d \right) \\
%&\Longleftrightarrow \exists \exists [ \dots ]\\
%&\Longleftrightarrow \exists [ \dots ]
\end{align}
Since only the first and the third quantifiers are unbounded, this is equivalent to a $\Pi_1$ statement. Then the property $\textrm{VCdim}(\mathcal{H}) < \infty$, equivalent to $(\exists d) [ \textrm{VCdim}(\mathcal{H}) < d ]$,
%\begin{align*}
%%\textrm{VCdim}(\mathcal{H}) < \infty &\Longleftrightarrow 
%(\exists d) [ \textrm{VCdim}(\mathcal{H}) < d ] %\\
%%&\Longleftrightarrow 
%%(\exists)(\forall)[\dots] 
%\end{align*}
is a $\Sigma_2$ property. This gives an upper bound on the arithmetical complexity for any computable family of hypothesis classes.% , on the complexity. %of the problem of finiteness of VC dimension.
% means that for any computable family of hypothesis classes, the problem of finiteness of VC dimension %(that is, PAC learnability) 
%is not more complicated than $\Sigma_2$.
\begin{fact}
The problem of PAC learnability for a computable family of hypothesis classes is no harder than $\Sigma_2$.
\end{fact}
Moreover, this bound is strict: as observed before by \citet{Sch99jcss} there are computable families of hypothesis classes such that the problem is $\Sigma_2$-complete. %For completeness, we provide a proof in Appendix \ref{sec:app}.
 %This was observed before by \citet[Theorem 4.1]{Sch99jcss}  
 The following proof is similar to that of \citet[theorem 4.1]{Sch99jcss} with reference to \citet{Weh90phd}, and is also implicit in \citet{Zha18unpubl}.
\begin{proposition}[\citealp{Sch99jcss}]\label{propo:pacsigma2}
There exists a computable family of hypothesis classes such that the problem of PAC learnability is $\Sigma_2$-complete.
\end{proposition}
\begin{proof}
%Our proof is similar to the sketch of \citet[Theorem 4.1]{Sch99jcss} with reference to \citet{Weh90phd}, and is also implicit in \citet{Zha18unpubl}. %\begin{proof}
We exhibit a computable family $\mathsf{H}= \{ \mathcal{H}_i \}_j$ for which the index set $\{ j \in \mathbb{N}: \mathcal{H}_j \textrm{ is learnable} \}$ is equal to the index set $\textrm{Fin} = \{ j \in \mathbb{N}: W_j \textrm{ is finite} \}$. The latter is well-known to be $\Sigma_2$-complete (see \citealp[p.\ 86]{Soa16}).

Let $\mathcal{H}_\mathrm{fin} = \{ h_i \}_{i \in \mathbb{N}}$ a computable enumeration of all hypotheses with finite support and $\{ W_j \}_{j \in \mathbb{N}}$ an enumeration of all c.e.\ sets. For every $j \in \mathbb{N}$ define c.e.\
\begin{align*}
N_j := \{ n \in \mathbb{N}: n \leq |W_j| \} = \{n \in \mathbb{N}: (\exists s)[ n \leq |W_{j,s}|]\},
\end{align*}
and let $\mathcal{H}_j := \{h_i: i \in N_j \}$. 
Then we have that $j \in \textrm{Fin}$ precisely if $\textrm{VCdim}(\mathcal{H}_j) < \infty$. Namely, if $j \in \textrm{Fin}$ then also $|N_j| = |\mathcal{H}|<\infty$ and $\textrm{VCdim}(\mathcal{H}_j) < \infty$. But if $j \notin \textrm{Fin}$ then $|N_j| = \mathbb{N}$ and $\mathcal{H}_j = \mathcal{H}_\mathrm{fin}$, so $\textrm{VCdim}(\mathcal{H}_j) = \infty$. 
\end{proof}
Next, we turn to SCPAC learnability. Recall its characterization, Theorem \ref{theorem:cpacefferm}, by the conjunction of finiteness of VC dimension and the computable implementability of ERM. We first introduce as a lemma an equivalent statement of the second conjunct, %(see Appendix \ref{sec:app} for a proof), 
that we can then express arithmetically to give us an upper bound. 
\begin{lemma}\label{lmm:ermceset}
For computable hypothesis class $\mathcal{H}$, $\textsc{ERM}_\mathcal{H}$ is computably implementable if and only if $B_{\mathcal{H}}:= \{S \in \mathcal{S}: (\exists h \in \mathcal{H})\left[ L_S(h) = 0 \right] \}$ is computable.
\end{lemma}
\begin{proof}
We have that $S \in B_{\mathcal{H}}$ precisely if $L_S\left(\textsc{ERM}_\mathcal{H}(S)\right)=0$, so it is immediate that if $\textsc{ERM}_\mathcal{H}$  is computable, then so is $B_{\mathcal{H}}$. Conversely, if the latter is computable, then the following procedure gives an algorithm for $\textsc{ERM}_\mathcal{H}$. For given $S = (x^n,y^n)$, for all $i\leq n$, $j \leq \binom{n}{i}$ define $z^n_{i,j}$ to be the $j$-th length-$n$ binary sequence that disagrees with $y^n$ on precisely $i$ positions. Now for the $i \leq n$ in increasing order, check for each defined $z^n_{i,j}$ whether $(x^n,z^n_{i,j}) \in B_{\mathcal{H}}$; as soon as this is the case for some $z^n_{i,j}$, start enumerating hypotheses in $\mathcal{H}$ until finding an $h$ with $L_S(h)=i$, and return this $h$. This procedure will always halt and return a hypothesis $\hat{h} \in \min_{h \in \mathcal{H}} L_S(h)$. 
\end{proof}
\begin{proposition}
The problem of SCPAC learnability for a computable family of hypothesis classes is no harder than $\Sigma_3$.
\end{proposition}
\begin{proof}
Let $\langle \cdot \rangle : \mathcal{S} \rightarrow \mathbb{N}$ be some computable 1-1 encoding of all finite samples onto the natural numbers. Given computable family $\mathsf{H} = \{ \mathcal{H}_j \}_j$, c.e.\ subset $B_i := \{ \langle S \rangle: S \in \mathcal{S} \ \& \ (\exists h \in \mathcal{H}_i)\left[ L_S(h) = 0 \right]   \} \subseteq \mathbb{N}$ is computable precisely if $B_{\mathcal{H}_i}$ is. %Moreover, we can computably obtain some index in the standard enumeration of c.e.\ sets such that $W_e = B'_\mathcal{H}$.
Since  (cf.\ \citealp[p.\ 83]{Soa16})
%The computability of $B_i$ can be expressed as
\begin{align*}
(\exists d)\left[ \overline{W}_d = B_i \right] &\Longleftrightarrow 
(\exists d)[B_i \cap W_d = \emptyset \wedge B_i \cap W_d = \mathbb{N}] \\
&\Longleftrightarrow 
(\exists d)\left[(\forall s)[B_{i,s} \cup W_{d,s} = \emptyset ]  \wedge (\forall x)(\exists s)[ x \in B_{i,s} \cap W_{d,s}] \right] \\
&\Longleftrightarrow 
(\exists)\left[(\forall)[\dots]  \wedge (\forall)(\exists)[\dots] \right] \\
&\Longleftrightarrow 
(\exists)(\forall)(\exists)[\dots],
\end{align*}
the computability of $B_i$ can be expressed as as a $\Sigma_3$ statement. But then the conjunction with the $\Sigma_2$ statement of finiteness of VC dimension is also a $\Sigma_3$ statement.
%As the arithmetical complexity of the finiteness of VC dimension is only $\Sigma_2$, the conjunction that characterizes SCPAC learnability can always be expressed as a $\Sigma_3$ statement. 
\end{proof}
Again, this bound is strict.
\begin{proposition}
There exists a computable family of hypothesis classes such that the problem of SCPAC learnability is $\Sigma_3$-complete.
\end{proposition}
\begin{proof}
We show for a family $\{ \mathcal{H}_j \}_{j \in \mathbb{N}}$ of classes of threshold functions that the question of SCPAC learnability is equivalent to the index set Rec $=\{j \in \mathbb{N}: W_j \textrm{ is computable} \}$, which is $\Sigma_3$-complete (\citealp[Theorem XVI]{Rog67}; also see \citealp[p.\ 88]{Soa16}). Recall that, for $i \in \mathbb{N}$, threshold function $h_i$ is defined by by $h(x)=1$ if and only if $x < i$. In addition, let $h_{\omega}$ be such that $h_{\omega}(x)=1$ for all $x$. From the standard enumeration $\{ W_j \}_{j \in \mathbb{N}}$ of the  c.e.\ sets, define $\mathcal{H}_j := \{ h_i : i \in W_j \} \cup \{h_\omega\}.$

Since each $\mathcal{H}_j$ has finite VC dimension, SCPAC learnability of $\mathcal{H}_j$ is equivalent to the computability of $B_{\mathcal{H}_j}$. Moreover, $B_{\mathcal{H}_j}$ is computable precisely if $\mathcal{H}_j$ is. Namely, starting with the right-to-left direction, to decide $h_i \in \mathcal{H}_j$ for $i \in \mathbb{N}$ (for $h_\omega$ the answer is always yes), it is enough to ask whether $((i,1),(i+1,0))\in B_{\mathcal{H}_j}$. Conversely, to decide $S \in B_{\mathcal{H}_j}$, we can distinguish four cases. First, if $y=1$ for all $(x,y) \in S$, then $S \in B_{\mathcal{H}_j}$ because $L_S(h_\omega)=0$. Second, if there are $(x,0),(x',1) \in S$ with $x<x'$ then $S \notin B_{\mathcal{H}_j}$. Third, if $y=0$ for all $(x,y) \in S$ then take the smallest $x_0$ with $(x,y)\in S$; now $S \in B_{\mathcal{H}_j}$ precisely if $h_x \in \mathcal{H}$ for some $x<x_0$. Otherwise, take the $(x_0,y_0),(x_1,y_1) \in S$ with $x_0<x_1$ and $y_0=1,y_1=0$ that have smallest difference $|x_0-x_1|$; now $S \in B_{\mathcal{H}_j}$ precisely if $h_x \in \mathcal{H}$ for some $x$ with $x_0 \leq x<x_1$. 

In sum, we have that $j \in \mathrm{Rec}$ iff  $\mathcal{H}_j$ is computable iff $\mathcal{H}_j$ is SCPAC learnable.
\end{proof}
If Question \ref{q:cpac=scpac} has a negative answer then the notions of CPAC and SCPAC learnability coincide, and we also have the complexity of the former. Otherwise, we need some different arithmetical characterization for CPAC learning. Similarly, to find the complexity of improper (S)CPAC learnability, we first need an arithmetical characterization of this notion (which would follow from a negative answer to Question \ref{q:cocb=imprcpac}).

\section{Conclusion and discussion}\label{sec:concl}
In the first part of this paper, we made progress on the main open problems concerning computable PAC (CPAC) learning: to give characterizations of (im)proper CPAC learnability. We gave a characterization of proper strong CPAC (SCPAC) learning, that is an effective version of the fundamental theorem of PAC learning; and we confirmed the conjecture that there are decidably representable PAC learnable classes that are not even improperly CPAC learnable. We leave as open questions whether every CPAC learnable class is already SCPAC learnable (in which case we already have a characterization of CPAC learnability) and whether every improperly CPAC learnable class is extendable to a properly CPAC learnable class (in which case we have a characterization of improper CPAC learnability). A natural further question of characterization concerns the notion of nonuniform CPAC learning \citep{Sol08arxiv,aablu20alt}, including a strong variant.%, where the sample complexity (in this case also a function on the hypothesis class) is computable.

In the second part, we investigated undecidability of (computable) PAC learning. We gave a basic argument to uncover both undecidability of learnability decision problems and the independence of ZFC of learnability, and we initiated a study of the arithmetical complexity of notions of learnability. Future characterizations of notions of learnability (e.g., of improper (S)CPAC learning or nonuniform (S)CPAC learning) also unlock the question of their arithmetic complexity.

What do our observations about undecidability mean for the motivating claim of \citet{aablu20alt}, that the ingredient of computability rules out ``independence of ZFC results of the type shown in \citet{BHMSY17arxiv,BHMSY19nmi}''? Proposition \ref{propo:undecindep} does state that for infinitely many particular RER $\mathcal{H}$ the learnability of $\mathcal{H}$ is  independent of ZFC (provided ZFC is arithmetically sound). We did not exhibit any \emph{particular} such class, but this is also not hard to do (recall Example \ref{ex:caroindzfc} of the class $\mathcal{H}_\mathrm{ZFC}$ of \citealp{Car21arxiv}). Perhaps the main difference with the original result of Ben-David et al.\ is that undecidable learnability statements in the computable framework of Agarwal et al.\ are in the end all statements of first-order arithmetic. Ben-David et al.\ showed that the EMX learnability of a particular hypothesis class is equivalent to the continuum hypothesis CH---or at least to a weak version of the CH (see \citealp{Har19naw})---which is a more complex set-theoretical statement.

%what is the relation between the undecidability of notions of (computable) PAC learning we discussed in this paper, and the earlier result of \citet{BHMSY19nmi}? To start, the undecidable learnability statements in the computable framework of Agarwal et al.\  are statements of first-order arithmetic. Ben-David et al.\ showed that the EMX learnability of a particular hypothesis class is equivalent to the continuum hypothesis CH---or at least to a weak version of the CH (see \citealp{Har19naw})---which is a more complex statement (specifically, of third-order arithmetic). This is important for their conclusion that there is no combinatorial characterization of EMX learning, thus, that there exists ``no general dimension for learning'' \citep[p.\ 47]{BHMSY19nmi}. 

This is important for the conclusion of Ben-David et al.\ that there is no combinatorial characterization of EMX learning, thus, that there exists ``no general dimension for learning'' \citeyearpar[p.\ 47]{BHMSY19nmi}. They write that a combinatorial ``dimension for learning'' (like VC dimension for PAC learning) is a ``finite character property'' (defined as ZFC-provably equivalent to a bounded formula in the language of set theory, or $\Delta_0$ in the L\'evy hierarchy; see \citealp{Jec03}) that does not vary over different models of ZFC (pp.\ 47--48). On a closer look \citeyearpar[p.\ 14]{BHMSY17arxiv}, Ben-David et al.\ restrict attention to a class of models of ZFC such that $\Delta_0$ properties have the same truth value in each model (these properties are ``absolute,'' in particular, for the class of transitive models of ZFC; see again \citealp{Jec03}). Under this restriction, ``loosely speaking, PAC learnability does not depend on the specific model of set theory,'' whereas ``EMX learnability heavily depends on the cardinality of the continuum'' and (provided ZFC is consistent) disagreeing models of ZFC ``are known to exist'' \citeyearpar[p.\ 14--15]{BHMSY17arxiv}. 

% More precisely, a finite character property is defined as ZFC-provably equivalent to a bounded formula in the language of set theory (technically, $\Delta_1$ in the L\'evy hierarchy), and ``if $\mathcal{X}$ and $\mathcal{H} \subseteq 2^\omega$ are the same'' in model $M_0$ and submodel $M_1 \subseteq M_0$ of set theory, then the property holds for $\mathcal{X}$ and $\mathcal{H}$ in $M_0$ iff it does for $M_1$ . Therefore, ``loosely speaking, PAC learnability does not depend on the specific model of set theory,'' whereas ``EMX learnability heavily depends on the cardinality of the continuum'' and (provided ZFC is consistent) models of ZFC that disagree on this ``are known to exist'' (ibid.). 

Now Proposition \ref{propo:undecindep} does also directly imply (provided ZFC is arithmetically sound) that for infinitely many particular RER $\mathcal{H}$ there are different models of ZFC that disagree  on whether VCdim$(\mathcal{H}) < \infty$ (whether $\mathcal{H}$ is PAC learnable). However, such disagreeing models, that must involve nonstandard models of arithmetic, are excluded by the above restriction of models. Here we enter the slippery territory of questions of truth and existence in mathematics (some entries to the relevant literature are \citealp{Koe09incol,ButWal18,Ham20}).  Most scholars in the foundations of mathematics would indeed find it implausible to claim that there is no truth to the arithmetical matter of whether a certain RER $\mathcal{H}$ is PAC learnable (has finite VC dimension), just because this is not settled among all (nonstandard) models of arithmetic. Even if we cannot pin it down with first-order axioms, they would argue, we have a clear conception of the natural numbers as per the intended, standard model. % (see, e.g., \citealp{Fra05}). 
Things are much more contentious when it comes to set theory and the continuum hypothesis. While it is therefore \emph{more} plausible to make the analogous claim about the non-existence of a dimension concept for EMX learnability, Ben-David et al.\ do still commit here to a philosophical position that is hardly uncontroversial. %(for an entry to the relevant literature about truth and absolute undecidability in mathematics, see \citealp{Koe09incol,Ste14incol,Ham20}).

% In this case, however, this must be because they have different (nonstandard) submodels of arithmetic. If we restrict ourselves to models of ZFC that agree on first-order arithmetic, then all models agree on (computable) PAC learnability in the current framework, but not on EMX learnability. Whether this represents a plausible dividing line between characterizing properties that still or no longer ``exist'' is ultimately a question in the philosophy of mathematics, that we cannot address here (see for an entry to the discussion about truth and realism).

%\vspace{-3mm}

\small

\bibliographystyle{abbrvnatnoaddress}
%\bibliography{all}

\end{document}